\documentclass[letterpaper, 10 pt, conference]{ieeeconf}  

\IEEEoverridecommandlockouts                              

\usepackage[english]{babel}
\usepackage[utf8]{inputenc}
\usepackage[cmex10]{amsmath}
\usepackage{amssymb}

\usepackage{float}
\usepackage[pdftex]{graphicx}
\usepackage{amsfonts}
\usepackage{mathtools}
\usepackage{xcolor}
\usepackage{color}
\usepackage{verbatim}
\usepackage{graphicx}
\usepackage{caption}
\usepackage{subcaption}
\usepackage{tcolorbox}
\usepackage{soul}
\usepackage{xcolor}
\usepackage{tcolorbox}
\tcbuselibrary{raster}

\usepackage{float}
\usepackage{bm}
\usepackage{units}
\usepackage{array}
\usepackage{hyperref}
\usepackage{MnSymbol}
\usepackage{sidecap}

{}
\newtheorem{definition}{Definition}{}
\newtheorem{corollary}{Corollary}{}
{}
\newtheorem{theorem}{Theorem}{}
{}
\newtheorem{lemma}{Lemma}{}
\newtheorem{example}{Example}{}

\title{Hybrid Zonotopes Exactly Represent ReLU Neural Networks}

\author{Joshua Ortiz, Alyssa Vellucci, Justin Koeln, Justin Ruths
\thanks{The authors are with the Department of Mechanical Engineering, The University of Texas at Dallas. Correspondence to {\tt\small jruths@utdallas.edu}.}
}

\begin{document}

\maketitle
\thispagestyle{empty}
\pagestyle{empty}

\begin{abstract}
We show that hybrid zonotopes offer an equivalent representation of feed-forward fully connected neural networks with ReLU activation functions. Our approach demonstrates that the complexity of binary variables is equal to the total number of neurons in the network and hence grows linearly in the size of the network. We demonstrate the utility of the hybrid zonotope formulation through three case studies including nonlinear function approximation, MPC closed-loop reachability and verification, and robustness of classification on the MNIST dataset.
\end{abstract}

\section{Introduction}
We leverage the recently introduced concept of a hybrid zonotope to develop an equivalent representation of feed-forward fully connected neural networks whose activation functions are rectified linear activation functions (ReLU) \cite{bird2021hybrid}. We show that this exact and analytic expression of the input-output mapping of the neural network enables verification and robustness quantification for general neural networks and, when neural networks are embedded in closed-loop control applications, facilitates closed-loop reachability analysis and safety guarantees. This work was developed separately from, and in parallel to, a recently posted article \cite{zhang2022reachability}. The overall goal of our paper and use of hybrid zonotopes is similar, however, we show our approach leads to linear growth in complexity, whereas the complexity in \cite{zhang2022reachability} grows exponentially.

To verify neural networks, a wide range of methods have been proposed to upper and lower bound the output of neural networks. In \cite{weng2018towards}, preactivation upper and lower bounds are calculated and propagated layer-by-layer to calculate an overall Lipschitz constant bound for the entire network. Due to their speed these iterative approaches continue to be used (e.g., \cite{shi2022efficiently}). However, it requires that bounds be recomputed if the input set changes and the iterative bounding erases any memory that exists between the layers. Semi-definite programming relaxations have been used to regain the layer-to-layer history in a computationally efficient manner \cite{fazlyab2019efficient} and mixed-integer linear programming (MILP) has been used to compute exact Lipschitz bounds \cite{jordan2020exactly}. In the context of control systems, closed-loop reachability has been investigated with similar methods \cite{hashemi2021certifying,everett2021reachability}.

Hybrid zonotopes are a generalization of constrained zonotopes, which were introduced several years earlier \cite{scott2016constrained}. Like our paper here, authors have leveraged constrained zonotopes \cite{chung2021constrained,zhang2022safety} or polytopes \cite{xiang2018reachable} to propagate sets through approximated ReLU activation functions to compute bounds on the output reachability of neural networks. \cite{wong2018provable} demonstrates how these analytic expression of relaxations of ReLU neural networks can be used to, for example, train networks to be more robust to adversarial attack. Star sets offer similar expressiveness as hybrid zonotopes and have also been used to provide approximate and exact reachability of feed-forward ReLU neural networks \cite{tran2019star}. However, like \cite{zhang2022reachability}, the computational complexity of the mixed-integer program (number of binary variables) grows exponentially. 

At the heart of our approach is the ability of hybrid zonotopes to maintain an exact analytic expression from input to output. Crucially, this expression preserves how previous layers impact the propagation of the sets without iterative calculations or overapproximation. In contrast to MILP approaches (e.g., \cite{tjeng2017evaluating}), which also provide exact input-output relationships for ReLU neural network certification, the hybrid zonotope representation, due to the memory feature of zonotopes, enables it to be used in a variety of ways, including for closed loop analysis.


\textbf{Notation:} Matrices are denoted by uppercase letters, e.g., $G\in\mathbb{R}^{n\times n_g}$, and sets by uppercase calligraphic letters, e.g., $\mathcal{Z}\subset\mathbb{R}^{n}$. Vectors and scalars are denoted by lowercase letters. The $n$-dimensional unit hypercube is denoted by $\mathcal{B}_{\infty}^n=\left\{x\in\mathbb{R}^{n}~|~\|x\|_{\infty}\leq1\right\}$. The set of all $n$-dimensional binary vectors is denoted by $\{-1,1\}^{n}$. Matrices and vectors of all $0$ and $1$ elements are denoted by $\mathbf{0}$ and $\mathbf{1}$, respectively, of appropriate dimension. The Kronecker product of matrices $ A \in 
 \mathbb{R}^{m \times n} $ and $ B \in 
 \mathbb{R}^{p \times q} $ is given as $ A \otimes B \in \mathbb{R}^{pm \times qn}$.

Given the sets $\mathcal{Z},\mathcal{W}\subset\mathbb{R}^{n},\:\mathcal{Y}\subset\mathbb{R}^{m}$, and matrix $R\in\mathbb{R}^{m\times n}$, the linear mapping of $\mathcal{Z}$ by $R$ is $R\mathcal{Z}=\{Rz~|~z\in\mathcal{Z}\}$, the Minkowski sum of $\mathcal{Z}$ and $\mathcal{W}$ is $\mathcal{Z}\oplus\mathcal{W}=\{z+w~|~z\in\mathcal{Z},\:w\in\mathcal{W}\}$, the generalized intersection of $\mathcal{Z}$ and $\mathcal{Y}$ under $R$ is $\mathcal{Z}\cap_R\mathcal{Y}=\{z\in\mathcal{Z}~|~Rz\in\mathcal{Y}\}$, and the Cartesian product of $\mathcal{Z}$ and $\mathcal{Y}$ is $\mathcal{Z}\times\mathcal{Y}=\{(z,y)|~z\in\mathcal{Z},\:y\in\mathcal{Y}\}$.

\section{Hybrid Zonotopes} \label{sec:HybZono}
\begin{definition} \label{def-hybridZono} \cite[Def. 3]{bird2021hybrid}
The set $\mathcal{Z}_h\subset\mathbb{R}^n$ is a \emph{hybrid zonotope} if there exists $G^c\in\mathbb{R}^{n\times n_{g}}$, $G^b\in\mathbb{R}^{n\times n_{b}}$, $c\in\mathbb{R}^{n}$, $A^c\in\mathbb{R}^{n_{c}\times n_{g}}$, $A^b\in\mathbb{R}^{n_{c}\times n_{b}}$, and $b\in\mathbb{R}^{n_c}$ such that
    \begin{equation}\label{def-eqn-hybridZono}
        \mathcal{Z}_h = \left\{ \left[G^c \: G^b\right]\left[\begin{smallmatrix}\xi^c \\ \xi^b \end{smallmatrix}\right]  + c\: \middle| \begin{matrix} \left[\begin{smallmatrix}\xi^c \\ \xi^b \end{smallmatrix}\right]\in \mathcal{B}_\infty^{n_{g}} \times \{-1,1\}^{n_{b}}, \\ \left[A^c \: A^b\right]\left[\begin{smallmatrix}\xi^c \\ \xi^b \end{smallmatrix}\right] = b \end{matrix} \right\}\:.
\end{equation}
\end{definition}

A hybrid zonotope is the union of $2^{n_b}$ constrained zonotopes corresponding to the possible combinations of binary factors, $\xi^b$. The hybrid zonotope is given in \textit{Hybrid Constrained Generator-representation} and the shorthand notation of $\mathcal{Z}_h=\langle G^c,G^b,c,A^c,A^b,b\rangle\subset\mathbb{R}^n$ is used to denote the set given by \eqref{def-eqn-hybridZono}. Continuous and binary \textit{generators} refer to the columns of $G^c$ and $G^b$, respectively. A hybrid zonotope with no binary generators is a constrained zonotope, $\mathcal{Z}_c=\langle G,c,A,b\rangle\subset\mathbb{R}^n$, and a hybrid zonotope with no binary generators and no constraints is a zonotope, $\mathcal{Z}=\langle G,c\rangle\subset\mathbb{R}^n$. Identities and time complexity of linear mappings, Minkowski sums, generalized intersections, and generalized half-space intersections are reported in \cite[Section 3.2]{bird2021hybrid}. An identity and time complexity for Cartesian products is given in \cite{bird2022dissertation}. Preliminary methods for removing redundant generators and constraints of a hybrid zonotope were reported in \cite{bird2021hybrid} and further developed in \cite{bird2022dissertation}.

\begin{example}[Zonotope Memory (see also \cite{kochdumper2020utilizing})]
If we consider a dynamic system $x_{k+1}=Ax_k+Bu_k$, with $x_0\in\mathcal{X}_0 = \langle G_x^c,G_x^b,c_x,A_x^c,A_x^b,b_x\rangle$ and $u_0\in \mathcal{U}=\langle G_u^c,G_u^b,c_u,A_u^c,A_u^b,b_u\rangle$, then by the linear mapping and Minkowski sum identities, 
\begin{equation}
\begin{aligned}
    x_1\in \mathcal{X}_1 = \Bigg\langle [AG_x^c \; &BG_u^c], [AG_x^b \; BG_u^b], Ac_x + Bc_u, \\ &\begin{bmatrix}A_x^c&\\&A_u^c\end{bmatrix} ,\begin{bmatrix}A_x^b&\\&A_u^b\end{bmatrix} , \begin{bmatrix}b_x\\b_u\end{bmatrix} \Bigg\rangle.
\end{aligned}
\end{equation}
The effect of the linear mapping is encoded in the transformation of the continuous and binary \underline{generators}. The memory is, however, captured in the fact that the continuous and binary \underline{factors} that define the sets $\mathcal{X}_0$ and $\mathcal{U}$ are embedded in the definition of $\mathcal{X}_1$. To ignore this connection - what happens with iterative methods - is to erase this memory by treating the factors of $\mathcal{X}_1$ as new unrelated factors. This feature of memory is especially visible when creating extended vectors, e.g., 
\begin{equation*}
\begin{aligned}
    \begin{bmatrix} x_0 \\ x_1 \end{bmatrix} \in \Bigg\langle
        &\begin{bmatrix} G_x^c&\\AG_x^c & BG_u^c \end{bmatrix},
        \begin{bmatrix} G_x^b&\\AG_x^b & BG_u^b \end{bmatrix},
        \begin{bmatrix} c_x\\Ac_x + Bc_u \end{bmatrix}, \\
        &\hspace{10mm}\begin{bmatrix}A_x^c&\\A_x^c&\\&A_u^c\end{bmatrix}, 
        \begin{bmatrix}A_x^b&\\A_x^b&\\&A_u^b\end{bmatrix},
        \begin{bmatrix}b_x\\b_x\\b_u\end{bmatrix}
    \Bigg\rangle.
\end{aligned}
\end{equation*}
Specifically, the continuous and binary factors that specify the hybrid zontope $\mathcal{X}_1$ are $[(\xi_x^c)^\top \; (\xi_u^c)^\top]^\top$ and $[(\xi_x^b)^\top \; (\xi_u^b)^\top]^\top$, respectively, where $\xi_x^c$ and $\xi_x^b$ specify $\mathcal{X}_0$ and $\xi_u^c$ and $\xi_u^b$ specify $\mathcal{U}$. There are no new factors introduced specific to $\mathcal{X}_1$ - all are inherited from $\mathcal{X}_0$ or $\mathcal{U}$. In this case the redundant constraints can be dropped.
\end{example}

\section{Neural Network Representation as a Hybrid Zonotope}

Consider an $L$-layered feed-forward fully-connected neural network $f: \mathbb{R}^n \rightarrow \mathbb{R}^m $ mapping inputs $ x \in \mathbb{R}^n $ to outputs $ y = f(x) \in \mathbb{R}^m $ such that 
\begin{align}
    \begin{split}
        x^0 &= x, \\
        x^{\ell+1} &= \phi(W^{\ell} x^{\ell} + b^{\ell}), \qquad {\ell} \in \{0,\cdots,L-1\}, \\
        y = f(x) &= W^{L} x^{L} + b^{L},
        \end{split}
\end{align}
where $ W^{\ell} \in \mathbb{R}^{n_{\ell+1} \times n_{\ell}} $ and $ b^{\ell} \in \mathbb{R}^{n_{\ell+1}} $ are the weight matrix and bias vector between layers $\ell$ and $\ell+1$, with $n_0 = n$ and $n_{L+1} = m $. For this paper, all activation functions $\phi$ are ReLU functions that operate element-wise, i.e., for the pre-activation vector $ v^{\ell+1} = W^{\ell} x^{\ell} + b^{\ell} \in \mathbb{R}^{n_{\ell+1}}$, 
\begin{equation}
    \phi(v^{\ell}) = [\varphi(v_1^{\ell}) \; \cdots  \; \varphi(v^{\ell}_{n_{\ell}})]^\top,
\end{equation}
where the ReLU function is defined as $\varphi(v_i) = \max\{0,v_i\} $.

The goal of this paper is to show that such a feed-forward neural network with only ReLU activation units can be exactly - and analytically - represented by a hybrid zonotope. Core to our approach is the idea that \textit{sets} (including hybrid zonotopes) can be used to describe \textit{functions}. We can accomplish this by expressing an extended set that expresses the input-to-output relationship. Given the input set $\mathcal{X}\subset\mathbb{R}^n$ and output set $\mathcal{Y}=\{f(x) : x\in\mathcal{X} \}\subset\mathbb{R}^m$, the function $f$ over the input set $\mathcal{X}$ can be expressed by the set $ \mathcal{F} \subset \mathbb{R}^{n+m} $ of extended vectors $ [ x^\top \; y^\top ]^\top \in \mathbb{R}^{n+m} $ such that $ \mathcal{F} = \{  [ x^\top \; y^\top ]^\top : y = f(x), x \in \mathcal{X} \}$. Zonotope-based sets especially enable function representations through the ``memory'' engendered by the factors. The crucial connection is that some of the factors that define the output zonotope set are factors derived from the input zonotope.

\begin{figure}[t]
     \centering
     \includegraphics[width=0.85\linewidth]{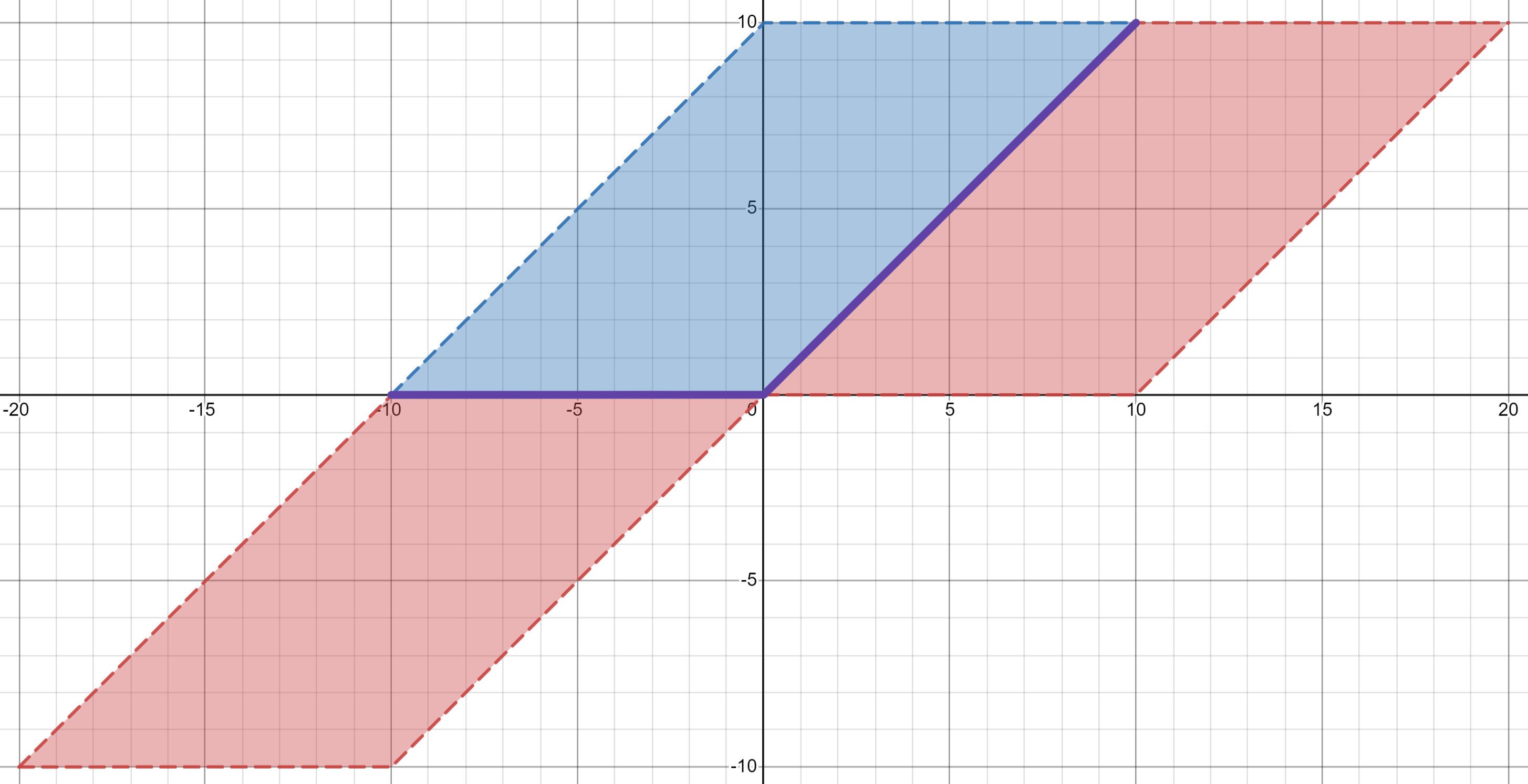}
     \caption{The intersection of the blue zonotope, $\mathcal{Z}_1$, and red hybrid zonotope, $\mathcal{Z}_2 \cup \mathcal{Z}_3$, is equivalent to the ReLU function (purple) over the interval $[-a, a]$ (here shown for $a=10$).}
    \label{fig:relu}
\end{figure}

To begin, we represent a single ReLU activation function $ x_i = \varphi(v_i) $ as a set of points $ [ v_i \; x_i ]^\top \in \mathbb{R}^{2} $ over a predetermined domain $ v_i \in [-a, a] $. We assume $ a > 0 $ is chosen large enough to capture the largest anticipated absolute value of the input $ v_i $ such that $ |v_i|\leq a $. 

\begin{lemma} \label{lem:relu}
    The set $ \Phi \subset \mathbb{R}^2 $ of points satisfying the ReLU activation function over the domain $ \mathcal{D}_i = [-a, a] $ can be exactly represented as a hybrid zonotope with $ n_g = 4 $ continuous generators, $ n_b = 1 $ binary generator, and $ n_c = 2 $ constraints.
\end{lemma}

\begin{proof}
    To satisfy the ReLU function for $ v_i \in [-a, a] $, $ \Phi = \{ [ v_i \; x_i ]^\top : x_i = \max\{0,v_i\}, |v_i|\leq a \}$. Note that $ \Phi = \Phi_1 \cup \Phi_2 $, where $ \Phi_1 = \{ [ v_i \; x_i ]^\top : -a \leq v_i \leq 0, x_i = 0 \} $ and $ \Phi_2 = \{ [ v_i \; x_i ]^\top : 0 \leq v_i \leq a, x_i = v_i \} $.  To represent this set, consider the zonotopes 
    \begin{equation}
    \begin{aligned}
    \mathcal{Z}_1 &= \left\langle \begin{bmatrix} \nicefrac{a}{2} & \nicefrac{a}{2} \\ \nicefrac{a}{2} & 0 \end{bmatrix}, \begin{bmatrix} 0 \\ \nicefrac{a}{2} \end{bmatrix} \right\rangle, \; 
    \mathcal{Z}_2 = \left\langle \begin{bmatrix} \nicefrac{a}{2} & \nicefrac{a}{2} \\ \nicefrac{a}{2} & 0 \end{bmatrix}, \begin{bmatrix} -a \\ -\nicefrac{a}{2} \end{bmatrix} \right\rangle, \\ 
    \mathcal{Z}_3 &= \left\langle \begin{bmatrix} \nicefrac{a}{2} & \nicefrac{a}{2} \\ \nicefrac{a}{2} & 0 \end{bmatrix}, \begin{bmatrix} a \\ \nicefrac{a}{2} \end{bmatrix} \right\rangle.
    \end{aligned}
    \end{equation}
    As shown in Figure \ref{fig:relu}, $\mathcal{Z}_1$, $\mathcal{Z}_2$, and $\mathcal{Z}_3$ are the same parallelogram with shifted centers. Additionally, Figure \ref{fig:relu} shows that $\mathcal{Z}_1 \cap \mathcal{Z}_2 = \Phi_1 $ and $\mathcal{Z}_1 \cap \mathcal{Z}_3 = \Phi_2 $. Therefore, $ \Phi = (\mathcal{Z}_1 \cap \mathcal{Z}_2) \cup (\mathcal{Z}_1 \cap \mathcal{Z}_3) = \mathcal{Z}_1 \cap (\mathcal{Z}_2 \cup \mathcal{Z}_3)$. Note that $ \mathcal{Z}_2 \cup \mathcal{Z}_3 $ can be directly expressed as the unconstrained hybrid zonotope
    \begin{equation}
        \mathcal{Z}_2 \cup \mathcal{Z}_3 = \left\langle \begin{bmatrix} \nicefrac{a}{2} & \nicefrac{a}{2} \\ \nicefrac{a}{2} & 0 \end{bmatrix}, \begin{bmatrix} a \\ \nicefrac{a}{2} \end{bmatrix}, \begin{bmatrix} 0 \\ 0 \end{bmatrix}, [\,], [\,], [\,] \right\rangle.
    \end{equation}
    Then, using the identities for the intersection operation in \cite{bird2021hybrid} (see Proposition 7),  
    \begin{align}
        \Phi &= \mathcal{Z}_1 \cap (\mathcal{Z}_2 \cup \mathcal{Z}_3),\\
        &= \Bigg\langle \begin{bmatrix} \nicefrac{a}{2} & \nicefrac{a}{2} & 0 & 0 \\ \nicefrac{a}{2} & 0 & 0 & 0 \end{bmatrix}, \begin{bmatrix} 0 \\ 0 \end{bmatrix}, \begin{bmatrix} 0 \\ \nicefrac{a}{2} \end{bmatrix}, \\
        &\hspace{15mm}\begin{bmatrix} \nicefrac{a}{2} & \nicefrac{a}{2} & -\nicefrac{a}{2} & -\nicefrac{a}{2} \\ \nicefrac{a}{2} & 0 & -\nicefrac{a}{2} & 0 \end{bmatrix}, \begin{bmatrix} -a \\ -\nicefrac{a}{2} \end{bmatrix}, \begin{bmatrix} 0 \\ -\nicefrac{a}{2} \end{bmatrix} \Bigg\rangle,\nonumber
    \end{align}
    with $ n_g = 4 $ continuous generators, $ n_b = 1 $ binary generator, and $ n_c = 2 $ constraints. 
\end{proof}

We can now assemble the individual ReLU functions as a hybrid zonotope into a hybrid zonotope representation for the entire neural network.

\begin{theorem}
    A feed-forward ReLU neural network $f:\mathbb{R}^n\to\mathbb{R}^m$ with $n_N$ total ReLU activation functions (neurons) can be exactly represented as a hybrid zonotope with $ n_g = 4 n_N $ continuous generators, $ n_b = n_N $ binary generators, and $ n_c = 3 n_N $ constraints.
\end{theorem}
\begin{proof}
    The proof constructs the set $ \mathcal{F} \subset \mathbb{R}^{n+m} $ of extended vectors $ [ x^\top \; y^\top ]^\top \in \mathbb{R}^{n+m} $ such that $ \mathcal{F} = \{  [ x^\top \; y^\top ]^\top : y = f(x), x \in [-a\mathbf{1}_{n}, a\mathbf{1}_{n} ] \}$. Building on Lemma \ref{lem:relu}, we construct the vector-valued element-wise ReLU function to compose a full layer of the neural network and then connect the layers through the affine mapping of weights and biases. It is always possible to select the parameter $a$ large enough since the function $f$ is Lipschitz continuous over a bounded domain.
    
    For a single layer of the neural network, we can represent the vector of $ n_{\ell} $ ReLU activation functions $ x^{\ell} = \phi(v^{\ell}) $, from (3), as a set of vectors $ [ (v^{\ell})^\top \; (x^{\ell})^\top]^\top \in \mathbb{R}^{2 n_{\ell}} $ over the domain $ \mathcal{D} = [-a\mathbf{1}_{n_{\ell}}, a\mathbf{1}_{n_{\ell}}] $. Note that $ [ v^{\ell}_1 \; x^{\ell}_1 \; \cdots \; v^{\ell}_{n_{\ell+1}} \; x^{\ell}_{n_{\ell+1}} ]^\top \in \Phi \times \; \cdots \; \times \Phi \triangleeq \Phi^{n_{\ell}} \subset \mathbb{R}^{2 n_{\ell}} $ and that $ [ (v^{\ell})^\top \; (x^{\ell})^\top]^\top = T [ v^{\ell}_1 \; x^{\ell}_1 \; \cdots \; v^{\ell}_{n_{\ell+1}} \; x^{\ell}_{n_{\ell+1}} ]^\top $, where
    \begin{equation}
        T = \begin{bmatrix} \mathbf{I}_{n_{\ell}} \otimes [1 \quad 0] \\ \mathbf{I}_{n_{\ell}} \otimes [0 \quad 1] \end{bmatrix}.
    \end{equation}
    Therefore, $ [ (v^{\ell})^\top \; (x^{\ell})^\top]^\top \in T \Phi^{n_{\ell}} $ are the points where $ v^{\ell} \in [-a\mathbf{1}_{n_{\ell}}, a\mathbf{1}_{n_{\ell}}] $ and $ x^{\ell} = \phi(v^{\ell}) $. Thus the transformation $T$ reorders the inputs and outputs of the Cartesian product of individual ReLU sets into an overall input-output set for the entire layer. Note that the hybrid zonotope $ T \Phi^{n_{\ell}} $ has $ n_g = 4 n_{\ell} $ continuous generators, $ n_b = 1 n_{\ell} $ binary generators, and $ n_c = 2 n_{\ell} $ constraints based on the definition of the Cartesian product for hybrid zonotopes \cite{bird2022dissertation}. This process can be repeated for each of the $L$ hidden layers of the neural network. This leads to $n_g = 4(n_1+\dots+n_{L})$ continuous factors, $n_b = 1(n_1+\dots+n_{L})$ continuous factors, and $n_c = 3(n_1+\dots+n_{L})$ constraints. 
    
    The affine map $v^{\ell+1} = W^\ell x^\ell + b^\ell$ provides the connection between layers $\ell = 0,\dots,L-1$. For the hybrid zonotope sets, the inter-layer connections provide constraints between $x^0$ and $v^1$, $x^1$ and $v^2$, \dots, $x_{L-1}$ and $v_L$ (the final affine mapping $y = W^L x^L + b^L$ provides a transformation of the final layer to the output, but not additional constraints). Since $v^\ell$ has dimension $n_\ell$, these $L$ vector constraints provide $n_1+\dots+n_{L}$ additional constraints over the $L$ layers. This leads to a total of $n_c = 3(n_1+\dots+n_{L})$ constraints. 
\end{proof}

With the set $ \mathcal{F} $ defined, it is now possible to consider a general set of inputs $ \mathcal{X} \subset \mathbb{R}^n $, assuming $ \mathcal{X} \subseteq [-a\mathbf{1}_{n}, a\mathbf{1}_{n} ] $, and identify the corresponding set of outputs $ \mathcal{Y} \subset \mathbb{R}^m $ such that $ y \in \mathcal{Y} $ if and only if there exists $ x \in \mathcal{X} $ such that $ y = f(x) $. 

\begin{corollary} \label{cor:outputSet}
    Given a hybrid zonotope of inputs $ \mathcal{X} \subset \mathbb{R}^n $ and a hybrid zonotope $ \mathcal{F} \subset \mathbb{R}^{n+m} $ of points satisfying the input-output mapping of a feed-forward ReLU neural network with ReLU activation functions, the set of corresponding outputs $ \mathcal{Y} \subset \mathbb{R}^m $ is a hybrid zonotope with $ n_{g,y} = n_{g,x} + 4 n_N $ continuous generators, $ n_{b,y} = n_{b,x} + n_N $ binary generators, and $ n_{c,y} = n + n_{c,x} + 3 n_N $ constraints, assuming $ \mathcal{X} $ has $ n_{g,x} $ continuous generators, $ n_{b,x} $ binary generators, and $ n_{c,x} $ constraints and the neural network has a total of $n_N$ ReLU activation functions.
\end{corollary}
\begin{proof}
    Consider the vector $ [ x^\top \; y^\top ]^\top \in \mathbb{R}^{n+m} $. If $ x \in [-a\mathbf{1}_{n}, a\mathbf{1}_{n} ] $ and $ y = f(x) $, then $ [ x^\top \; y^\top ]^\top \in \mathcal{F} \subset \mathbb{R}^{n+m} $. Additionally, if $ x \in \mathcal{X} $, then $ [ x^\top \; y^\top ]^\top \in \mathcal{F} \cap_{[\mathbf{I}_n \; \mathbf{0}_{n \times m}]} \mathcal{X} $. Finally, the output set $ \mathcal{Y} = [ \mathbf{0}_{m \times n} \; \mathbf{I}_m](\mathcal{F} \cap_{[\mathbf{I}_n \; \mathbf{0}_{n \times m}]} \mathcal{X}) $, which has $ n_{g,y} = n_{g,x} + 4 n_N $ continuous generators, $ n_{b,y} = n_{b,x} + n_N $ binary generators, and $ n_{c,y} = n + n_{c,x} + 3 n_N $ constraints based on the definition of the generalized intersection from \cite{bird2021hybrid} (see Proposition 2) and the fact that the projection does not change the number of generators or constraints.
\end{proof}

A powerful outcome of Corollary \ref{cor:outputSet} is that the hybrid zonotope representation of the neural network, $\mathcal{F}$ does not need to be recomputed or reformulated if the input domain $\mathcal{X}$ changes. This is particularly appealing in the context of inspecting robustness and variation in the mapping provided by the neural network. This notion will be used in the MNIST case study in Section \ref{sec:MNIST}.


\section{Closed-loop Reachability of Linear Systems under Neural Network Control}

Consider the discrete-time linear system $x_{k+1} = A x_k + B u_k$,
where $ x_k \in \mathbb{R}^n $ are the states and $ u_k \in \mathbb{R}^m $ are the inputs. We assume closed-loop control using a neural network such that $ u_k = f(x_k) $, resulting in the closed-loop dynamics
\begin{equation}\label{eq:closedLoopDynamics}
    x_{k+1} = A x_k + B f(x_k).
\end{equation}

\begin{figure*}[t]
    \centering
    \includegraphics[width=0.35\linewidth]{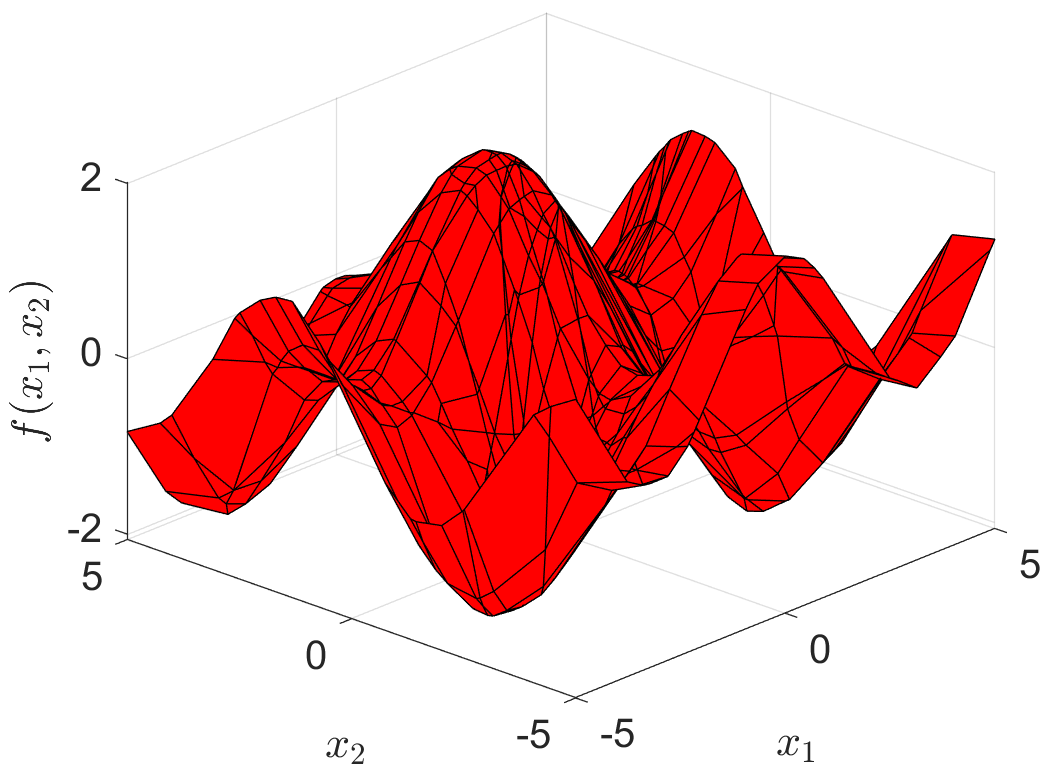}
    \hspace{0.5cm}
    {\includegraphics[width=0.25\linewidth]{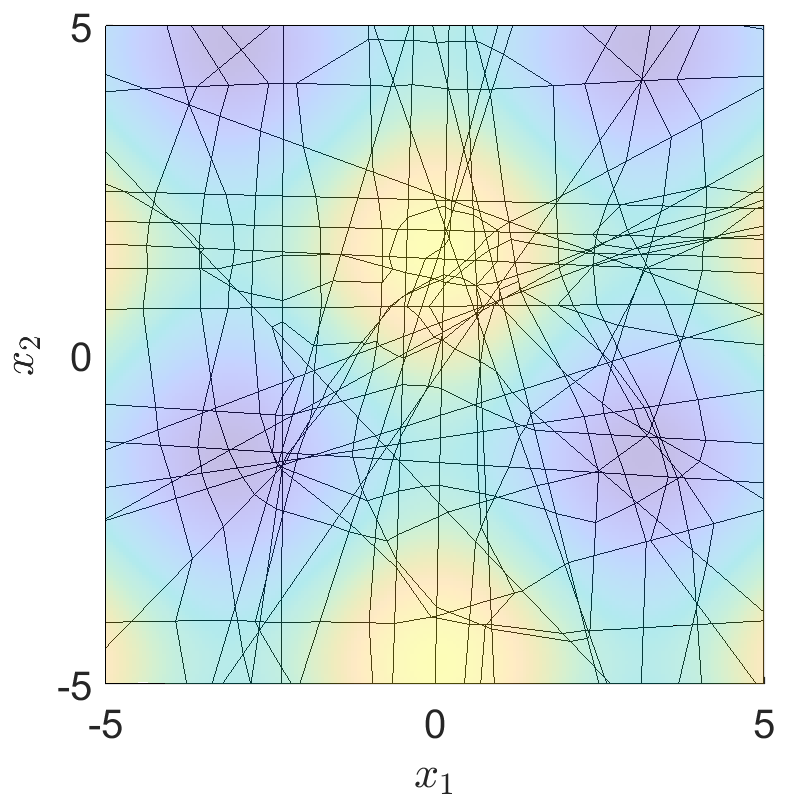}}
    \caption{(Left) Hybrid zonotope of the neural network that approximates the function $\cos(x_1)+\sin(x_2)$. (Right) Top-down view of the surface reveals the intricate faceting that corresponds to binary factors (equivalent to ReLU activations) being on/off. The colormap of the original function is overlaid to show how the faceting follows the contours.}
    \label{fig:sincos}
\end{figure*}

\begin{theorem}\label{thm:closedLoopReach}
    Assume the reachable set $ \mathcal{R}_k $ at time step $ k $ is represented as a hybrid zonotope with $ n_{g,k} $ continuous generators, $ n_{b,k} $ binary generators, and $ n_{c,k} $ constraints and that the neural network with $n_N$ ReLU activation functions is represented as a hybrid zonotope $\mathcal{F}$ with $ 4 n_N $ continuous generators, $ n_N $ binary generators, and $ 3 n_N $ constraints. Then the reachable set $ \mathcal{R}_{k+1}$ at time step $ k + 1 $ is a hybrid zonotope with $ n_g = n_{g,k} + 4 n_N $ continuous generators, $ n_b = n_{b,k} + n_N $ binary generators, and $ n_c = n + n_{c,k} + 3 n_N $ constraints.
\end{theorem}
\begin{proof}
    The proof is similar to that of \textbf{Corollary \ref{cor:outputSet}}. First, we can write \eqref{eq:closedLoopDynamics} as $ x_{k+1} = [A \; B] [ x_k^\top \; u^\top ]^\top $. Then we note that $ [ x_k^\top \; u^\top ]^\top \in \mathcal{F} \cap_{[\mathbf{I}_n \; \mathbf{0}_{n \times m}]} \mathcal{R}_k$ such that $ x_{k+1} \in [A \; B](\mathcal{F} \cap_{[\mathbf{I}_n \; \mathbf{0}_{n \times m}]} \mathcal{R}_k)$.
\end{proof}

The result of \textbf{Theorem \ref{thm:closedLoopReach}} can be repeated recursively starting from $ \mathcal{R}_0 = \mathcal{X}_0 $ to compute the reachable set $ \mathcal{R}_k$ at any time step $ k $. We omit the simple proof to satisfy space constraints.

\begin{corollary} \label{cor:reachableset}
     Assume the initial conditions set $\mathcal{R}_0$ is represented as a hybrid zonotope  with $ n_{g,0} $ continuous generators, $ n_{b,0} $ binary generators, and $ n_{c,0} $ constraints and that the neural network with $n_N$ ReLU activation functions is represented as a hybrid zonotope $\mathcal{F}$ with $ 4 n_N $ continuous generators, $ n_N $ binary generators, and $ 3 n_N $ constraints. Then the hybrid zonotope set representation complexity of the reachable set $ \mathcal{R}_k $ at time step $ k $ grows linearly with $ k $, resulting in $ n_{g,k} = n_{g,0} + 4 k n_N $ continuous generators, $ n_{b,k} = n_{b,0} + k n_N $ binary generators, and $ n_{c,k} = n_{c,0} + k(n+3 n_N) $ constraints.
\end{corollary}


In the following result, we demonstrate how the reachable sets over time can be stacked to provide additional interpretability with regard to the history of the control policy used. This result will be used in the context of the MPC case study in Section \ref{sec:MPC}.

\begin{corollary} \label{cor:reachability_stacked}
     Assume the initial conditions set $\mathcal{R}_0$ is represented as a hybrid zonotope  with $ n_{g,0} $ continuous generators, $ n_{b,0} $ binary generators, and $ n_{c,0} $ constraints and that the neural network with $n_N$ ReLU activation functions is represented as a hybrid zonotope $\mathcal{F}$ with $ 4 n_N $ continuous generators, $ n_N $ binary generators, and $ 3 n_N $ constraints. Then the hybrid zonotope set representation of the aggregate reachable set $[x_0^\top \; x_1^\top \; \cdots \; x_k^\top]^\top \in \mathcal{R}_{0:k} = \mathcal{R}_0 \times \mathcal{R}_1 \times \cdots \times \mathcal{R}_k $ has complexity $ n_{g,k} = n_{g,0} + 4 k n_N $ continuous generators, $ n_{b,k} = n_{b,0} + k n_N $ binary generators, and $ n_{c,k} = n_{c,0} + (n+3n_N)\sum_{\kappa=0}^k \kappa  $ constraints.
\end{corollary}
\begin{proof}
Following from Corollary \ref{cor:reachableset}, the hybrid zonotope $\mathcal{R}_{0:k}$ is the Cartesian product of sets $\mathcal{R}_\kappa$ for $\kappa = 1,2,\dots,k$ each with $ n_{g,\kappa} = n_{g,0} + 4 \kappa n_N $ continuous generators, $ n_{b,\kappa} = n_{b,0} + \kappa n_N $ binary generators, and $ n_{c,\kappa} = n + n_{c,0} + 3 \kappa n_N $ constraints. If these sets were unrelated, the resulting hybrid zonotope would have $\sum_{\kappa=0}^k n_{g,\kappa}$ continuous generators, $\sum_{\kappa=0}^k n_{b,\kappa}$ binary generators, and $\sum_{\kappa=0}^k n_{c,\kappa}$.

However, from the proof of Theorem \ref{thm:closedLoopReach}, $ \mathcal{R}_{\kappa+1} = [A \; B](\mathcal{F} \cap_{[\mathbf{I}_n \; \mathbf{0}_{n \times m}]} \mathcal{R}_\kappa$), hence $\mathcal{R}_{\kappa+1}$ and $\mathcal{R}_\kappa$ share common factors. Namely, all continuous and binary factors that characterize $\mathcal{R}_\kappa$ also contribute to characterize $\mathcal{R}_{\kappa+1}$. Thus the extended vector $[x_\kappa^\top \; x_{\kappa+1}^\top]^\top$ belongs to a hybrid zonotope with $\max(n_{g,\kappa},n_{g,\kappa+1})$ continuous and $\max(n_{b,\kappa},n_{b,\kappa+1})$ binary generators. The constraints do stack and become $n_{c,\kappa}+n_{c,\kappa+1}$.

Extending this pairwise relationship forward,  $\mathcal{R}_{0:k}$ has $\max(n_{g,0},\dots,n_{g,k})=n_{g,k}=n_{g,0}+4kn_N$ continuous generators, $\max(n_{b,0},\dots,n_{b,k})=n_{b,k}=n_{b,0}+4kn_N$ binary generators, and $n_{c,0}+(n+3n_N)\sum_{\kappa=0}^k \kappa$ constraints.
\end{proof}

\section{Applications \& Demonstrations}

The neural networks in this paper are feed-forward fully-connected neural networks with only ReLU activation units, and specified by the number of layers and width of each layer, e.g., [4,10,7,2] is a network with inputs in $\mathbb{R}^4$, outputs in $\mathbb{R}^2$, and with two hidden layers of 10 and 7 ReLU units, respectively. The networks are trained in MATLAB using the stochastic gradient descent optimizer with momentum (0.95) and 100 epochs. The hybrid zonotopes have been coded in MATLAB. Optimization problems have been solved using GUROBI \cite{optimization2021llc}. These examples are conducted on a laptop computer using one core of an 1.9 GHz Intel i7 processor and 16GB of RAM.

\begin{figure*}[t]
    \centering
    \includegraphics[width=0.4\linewidth]{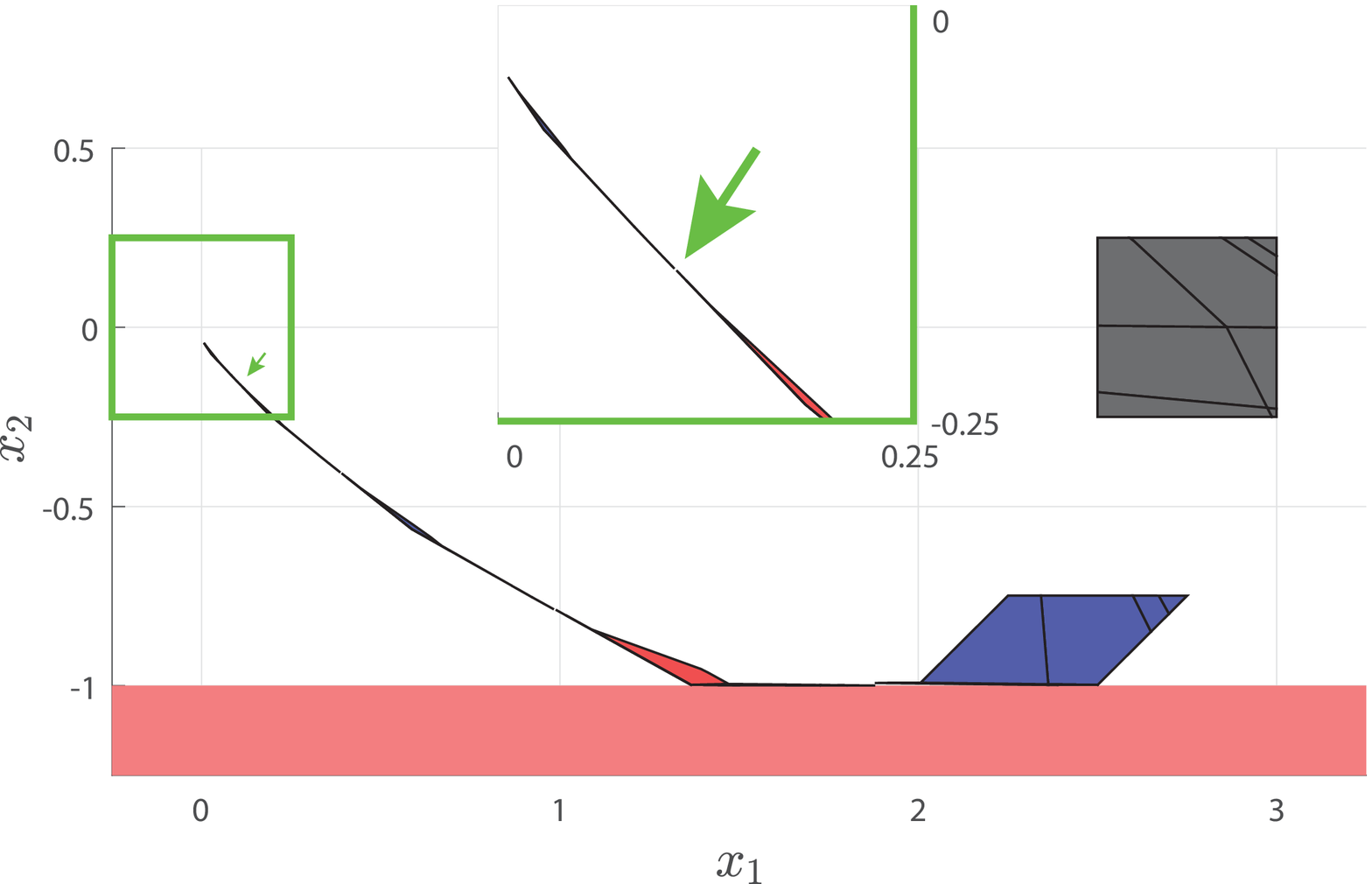}
    \hspace{0.1cm}
    \includegraphics[width=0.35\linewidth]{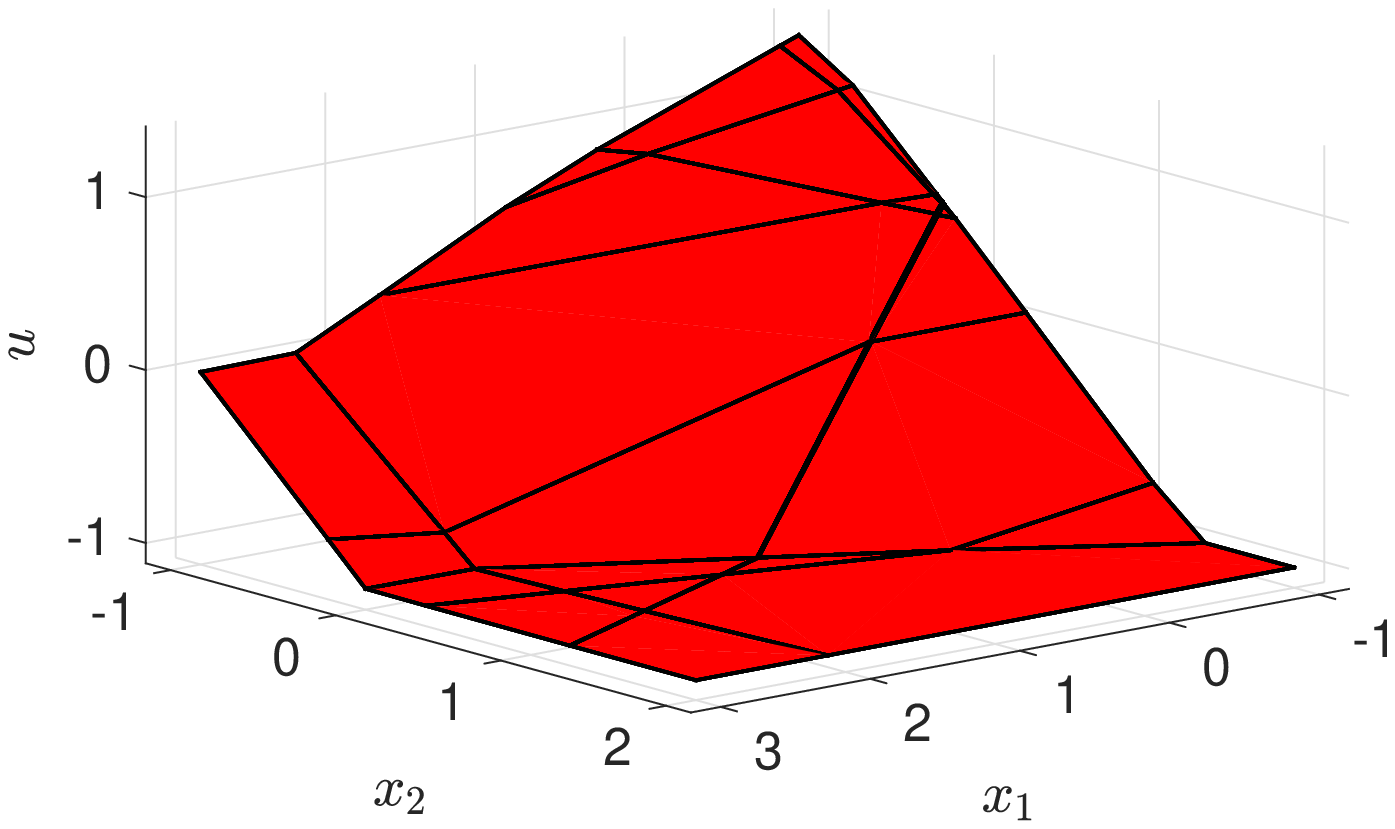}
    \caption{(Right) The hybrid zonotope representation of the neural network trained to learn an MPC policy to reach the origin while avoiding constraints, e.g., avoid $x\leq -1$. (Left) The evolution of the system state following the inputs produced by the neural network model. The initial state set at $k=0$ is shown in gray, and steps $k=1,\dots,5$ alternate color between blue and red. The inset plot shows (green arrow) the break between the $k=4$ and $k=5$ reachable set to show that the final state set sits within the goal set (green box).}
    \label{fig:MPC}
\end{figure*}

\subsection{Nonlinear Function Approximation} \label{sec:surface}

Over the input domain $[-5,5]\times[-5,5]$ we train a ReLU feed-forward neural network with layer sizes [2,20,10,10,1] to approximate the function $\cos(x_1)+\sin(x_2)$ by grid sampling 400 points.

There are $20+10+10=40=n_N$ total ReLU functions in the network and we build a hybrid zonotope $\mathcal{F}$ capturing the input-output mapping using the extended vector $[x^\top y^\top]^\top$. The hybrid zonotope has $n+4n_N = 2 + 4\times 40=162$ continuous generators, $n_N=40$ binary generators, and $3n_N=120$ constraints. Of the $2^{40}$ possible combinations of binary factors, 829 satisfy the constraints, leading to 829 constrained zonotopes which appear as facets of the surface plotted in Figure \ref{fig:sincos}. There is a one-to-one correspondence between these 829 feasible binary factor combinations and the combinations of the 40 ReLU units being active ($>0$) or inactive ($\leq 0$).

\subsection{Model Predictive Control Policy \& Closed-Loop Reachability Analysis} \label{sec:MPC}


We consider a double integrator system discretized with a sampling time of 1 second,
\begin{equation}
    x_{k+1} = \begin{bmatrix}1&1\\0&1\end{bmatrix}x_k + \begin{bmatrix}0.5\\1\end{bmatrix}u_k.
\end{equation}
As in \cite{zhang2022reachability,hu2020reach,everett2021reachability}, an MPC control policy is designed to stablize the system to the origin while respecting the state and input constraints, $x_k\in[-5,5]\times[-1,1]$ and $u_k\in[-1,1]$. A prediction horizon of 10 steps is used, with state and input weighting matrices $Q=I_2$ and $R=I$, the terminal region $\mathcal{O}_\infty^\text{LQR}$, and with terminal penalty matrix $P_\infty$, the solution of the discrete time algebraic Riccati equation. We grid the state space over $[-0.5,3]\times[-1,1]$ to produce 10,201 state-input pairs to train a neural network with layer sizes $[2,8,4,1]$.

The mapping of state-to-control provided by the neural network is visualized as a hybrid zonotope surface in three dimensions in Figure \ref{fig:MPC} (right). This hybrid zonotope, with $n=2$ and $n_N=4+8=12$, has 50 continuous factors, 12 binary factors, and 36 constraints. Each facet of the surface corresponds to a different combination of binary factors. Although there are $2^{12}$ different factor combinations, there are only 32 facets in the surface. The remaining binary factor combinations do not satisfy the linear constraints.

We now use the hybrid zonotope representation of the trained neural network to analyze the closed-loop reachability of the system, similar to the approach taken for an explicit MPC controller \cite{bird2022explicitmpc}. We consider an initial state (the gray square in Figure \ref{fig:MPC}), $x(0)\in[2.5,3]\times[-0.25,0.25]$. Following Corollary \ref{cor:reachability_stacked}, we construct the hybrid zonotope for the extended vector of states across time, $[x_0^\top \; \cdots \; x_5^\top]^\top$. 
Despite the 60 binary generators in the extended state hybrid zonotope, there are only 8 binary factor combinations (8 different sequences of linear control policies) that satisfy the constraints. Part of the utility of the extended state vector is that the initial set (gray square) is originally specified as an unconstrained zonotope with a center and two continuous generators. As seen in Figure \ref{fig:MPC}, the initial set is faceted by constraints at later time steps to reveal the sets of initial conditions that correspond to different sequences of linear feedback control policies.

Suppose that we wish to check that the system reaches the goal set $[-0.25,0.25]\times[-0.25,0.25]$ (green square in Figure \ref{fig:MPC}) for any initial condition in the gray region. Although this can be confirmed by visual inspection, it can also be rigorously ensured by evaluating set containment. The containment check of two hybrid zonotopes can be posed as a feasibility evaluation of a mixed-integer linear program.

\subsection{Classification Robustness on MNIST} \label{sec:MNIST}

The MNIST dataset is a canonical classification problem in which $28\times 28$ pixel images of handwritten numbers are classified into the digits 0-9. For clarity of presentation, we consider here the classification task to classify the digits ``1'' and ``7''. We train a ReLU feedfoward network with layer sizes [784,5,5,1] on a corpus of 13,007 images of the digits 1 and 7, in which the input is the vectorized image (stacking columns) and the output value $+1$ denotes a ``1'' and the output value $-1$ denotes a ``7''. The trained network achieves an accuracy of 99.5\% on a test bank of 2,163 images of the digits 1 and 7. Figure \ref{fig:MNIST} (left) shows the neural network output over the test image set, with outliers denoting images that are classified poorly and - in a handful of cases - incorrectly. Inset are example images of ``the 1 that looks most like a 7'' (worst 1) and ``the 7 that looks most like a 1'' (worst 7). These examples show that the output range over the images that are 1s spans nearly the entire interval [-1,1] (same is true for 7s) despite the overall good performance of 99.5\%. 


\begin{figure*}[t]
    \centering
    \includegraphics[width=0.45\linewidth]{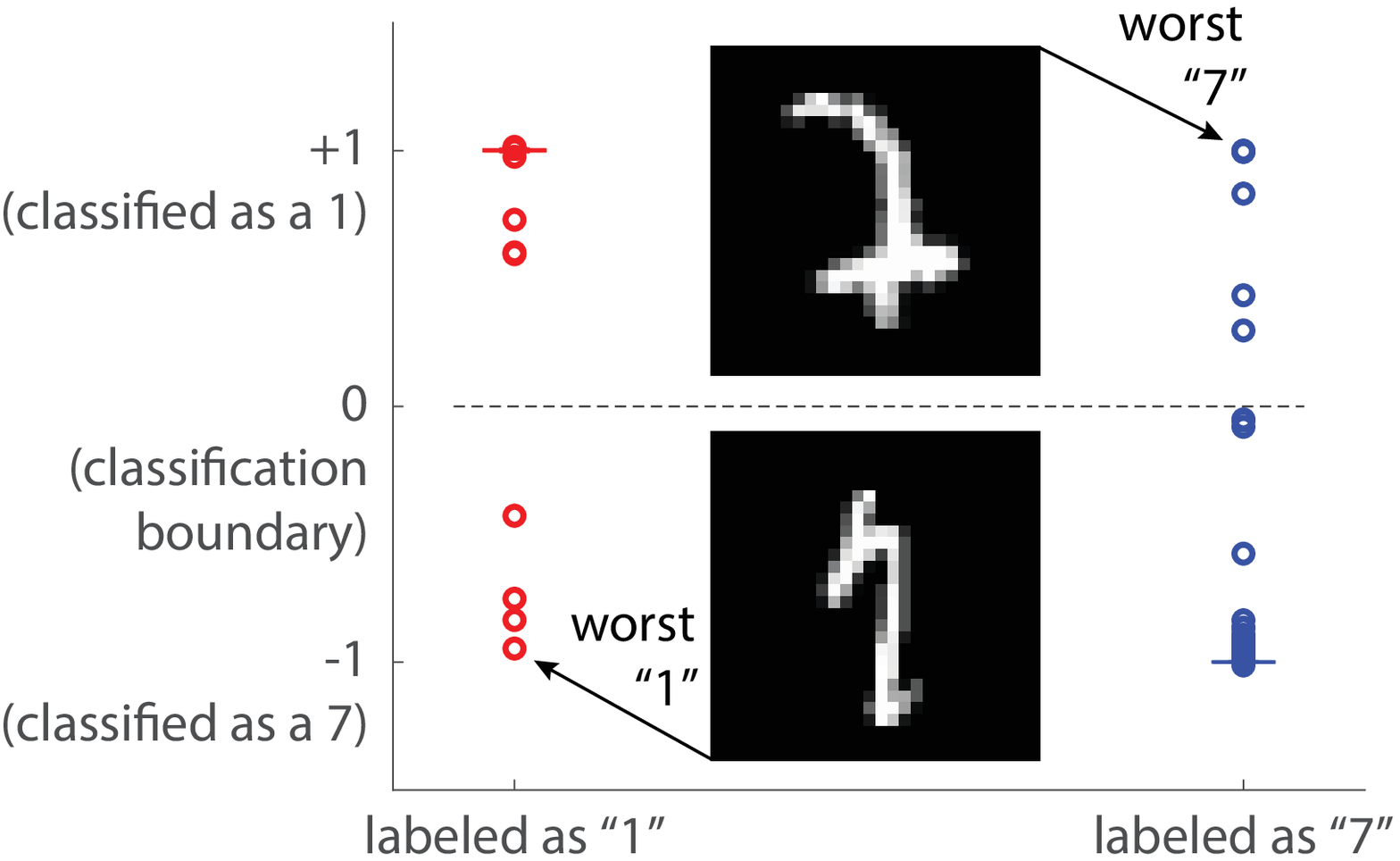}
    \hspace{0.5cm}
    \includegraphics[width=0.4\linewidth]{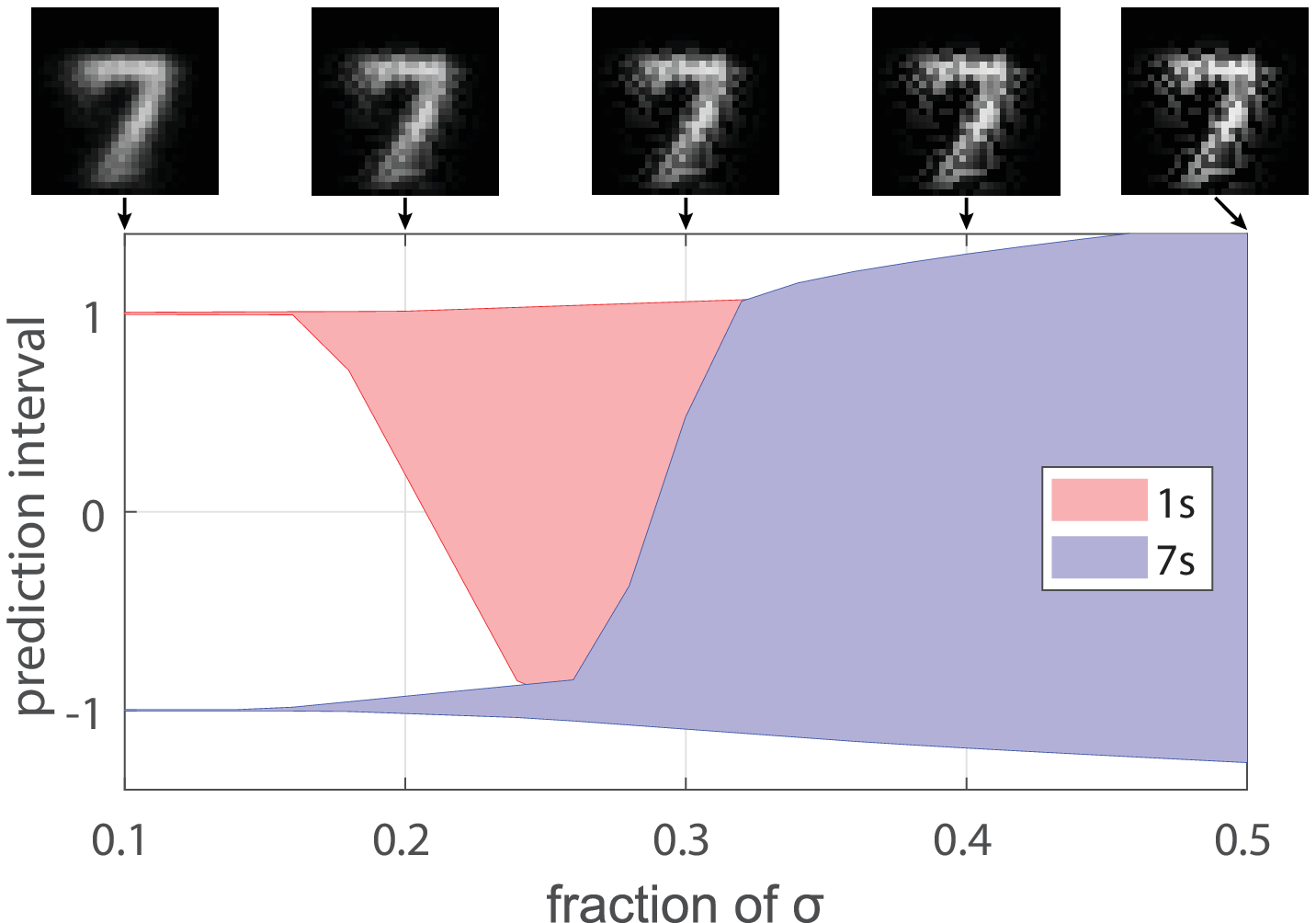}
    \caption{(Left) The trained neural network achieves classification accuracy of 99.5\% on the test image set, with outliers distributed across the entire output range $[-1,1]$. Inset plots show example images of worst case classification. (Right) The output range of the neural network, computed using the hybrid zonotope representation, over input sets centered at the nominal (mean) image of a ``1'' or a ``7'' and with expanding size moving left to right. On top, sample images from the 7s input set (defined in the text) from randomly generated binary factors at different fractions of $\sigma$.}
    \label{fig:MNIST}
\end{figure*}

The hybrid zonotope representation of the network allows us to explore this output range more rigorously as a function of the input space. We define the \textit{1s input space} as a zonotope whose center is the mean value of the vectorized images of the digit 1 (both training and testing). Defining $\sigma\in\mathbb{R}^{784}$ as the standard deviation of the same vectorized images, the (continuous) generator matrix of the 1s input set is $G_c = \text{diag}(\sigma)$. It is important to note that while the input set has been defined by the images, the input set is now the interval box ``containing'' one standard deviation away from the mean image in 784-dimensional space. Thus this 1s input set contains the entire continuous region around the nominal image of the digit 1. We follow the same procedure to produce the \textit{7s input space}. With these input spaces, the hybrid zonotope representing the neural network has $784+4\times 10=824$ continuous generators, $10$ binary generators, and $30$ constraints.

We now consider an adjustment to these input sets such that $G_c=\text{diag}(\alpha\sigma)$, where $0<\alpha\leq 1$ specifies a fractional value of the standard deviation. When $\alpha$ approaches zero, the input set approaches a single point, located at the center (mean) value. Given the classification accuracy of the neural network, this mean point will be reliably classified correctly, hence we expect the classification output range of the 1s input set with small $\alpha$ to be very near $+1$ and the output range of the 7s input set with small $\alpha$ to be very near $-1$. As $\alpha$ grows, we expect that the output range to grow since we observed that misclassifications occur in the test dataset. Figure \ref{fig:MNIST} (right) demonstrates how the hybrid zonotope representation of the neural network allows exact characterization of the upper and lower bounds of the output hybrid zonotope subject to the 1s input set and the 7s input set. The calculation of the upper and lower bounds along each dimension of a hybrid zonotope can be cast as a mixed-integer linear optimization problem.

Figure \ref{fig:MNIST} (right) demonstrates that although the classification accuracy is good on the test images, relatively small perturbations (in this case perturbations on the order of $0.3\sigma$) are sufficient to compromise the classification of some images. The sample images from the 7s input set shows that sample images from input sets corresponding to $\alpha=0.3$ are still human-identifiable as their correct digit. Being able to quantify the output bounds on a network over a given input domain is a powerful tool to probe the robustness of a neural network. This is especially true in the context of adversarial attacks on neural networks, in which small adjustments can lead to large changes in output value (e.g., classification outcome) \cite{eykholt2018robust}. 

\section{Conclusions}
In this paper, we have shown that it is possible to exactly represent a ReLU feed-forward neural network as a hybrid zonotope. This approach provides the ability to characterize variation and robustness of the mapping given by the neural network and offers rigorous ways to certify safety and reachability in closed-loop control application.

A direction for future work is to formulate the Lipschitz constant for the neural network directly from the hybrid zonotope representation. Also, although the binary factors scale linearly with the number of ReLU neurons in the network, focusing on ways to best leverage advancements in mixed-integer linear optimization solvers will unlock the ability to analyze larger and larger networks.

\section{ACKNOWLEDGMENTS}
We thank Neera Jain and Trevor Bird for sharing their MATLAB code for hybrid zonotopes.

\bibliographystyle{IEEEtran}
\bibliography{refs}

\end{document}